\documentclass[11pt,a4paper]{article}
\usepackage[hyperref]{acl2020}
\usepackage{times}
\usepackage{latexsym}
\usepackage{mathtools}
\usepackage{colonequals}
\usepackage{comment}
\usepackage[pdftex]{graphicx} 
\usepackage{multirow}
\usepackage{anyfontsize}
\usepackage{tabu}
\usepackage{booktabs}
\usepackage{amsmath,amssymb}
\usepackage{verbatim}
\usepackage{bm}
\usepackage{pgfplots}
\usepackage{enumitem}
\usepackage[skip=2pt]{caption} 
\usepackage{float}
\usepackage{tikz}
\usepackage{mathtools}
\usepackage{xcolor,soul}
\usepackage{scrextend}
\usepackage{float}
\usepackage{amsthm}

\DeclareMathOperator{\E}{\mathbb{E}}
\DeclareMathOperator{\loss}{\mathcal{L}}
\DeclareMathOperator{\rvx}{{\bf x}}
\DeclareMathOperator{\rvy}{{\bf y}}
\DeclareMathOperator{\rvzx}{{\bf z_x}}
\DeclareMathOperator{\rvzy}{{\bf z_y}}
\DeclareMathOperator{\rvc}{{\bf c}}
\DeclareMathOperator{\kl}{\text{D}_{\text{KL}}}
\DeclareRobustCommand{\hlpink}[1]{{\sethlcolor{pink}\hl{#1}}}
\DeclareRobustCommand{\hlyellow}[1]{{\sethlcolor{yellow}\hl{#1}}}
\DeclareRobustCommand{\hlcyan}[1]{{\sethlcolor{cyan}\hl{#1}}}
\DeclareRobustCommand{\hlgreen}[1]{{\sethlcolor{green}\hl{#1}}}
\newcommand{\bx}{\mathbf{x}}
\newcommand{\by}{\mathbf{y}}
\newcommand{\bz}{\mathbf{z}}
\newcommand{\bc}{\mathbf{c}}
\newtheorem*{Thm*}{Theorem}

\usepackage{url}

\usepackage{microtype}

\aclfinalcopy 
\def\aclpaperid{789} 

\title{Generating Diverse and Consistent QA pairs from Contexts with Information-Maximizing Hierarchical Conditional VAEs}

\author{Dong Bok Lee$^1$$^*$\thanks{* Equal contribution} \: Seanie Lee$^{1,3}$$^*$ \: Woo Tae Jeong$^3$ \: Donghwan Kim$^3$ \: Sung Ju Hwang$^1$$^,$$^2$ \\  
	KAIST$^1$,  AITRICS$^2$, 42Maru Inc.$^3$,  South Korea\\
	\texttt{\{markhi,lsnfamily02,sjhwang82\}@kaist.ac.kr}\\
	\texttt{\{wtjeong,scissors\}@42maru.com}}
	
\renewcommand\footnotemark{}
\date{}

\begin{document}
\pgfplotsset{width=0.95\linewidth}
\maketitle
\begin{abstract}
One of the most crucial challenges in question answering (QA) is the scarcity of labeled data, since it is costly to obtain question-answer (QA) pairs for a target text domain with human annotation. An alternative approach to tackle the problem is to use automatically generated QA pairs from either the problem context or from large amount of unstructured texts (e.g. Wikipedia). In this work, we propose a hierarchical conditional variational autoencoder (HCVAE) for generating QA pairs given unstructured texts as contexts, while maximizing the mutual information between generated QA pairs to ensure their consistency. We validate our {\bf Info}rmation Maximizing {\bf H}ierarchical {\bf C}onditional {\bf V}ariational {\bf A}uto{\bf E}ncoder ({\bf Info-HCVAE}) on several benchmark datasets by evaluating the performance of the QA model (BERT-base) using only the generated QA pairs (QA-based evaluation) or by using both the generated and human-labeled pairs (semi-supervised learning) for training, against state-of-the-art baseline models. The results show that our model obtains impressive performance gains over all baselines on both tasks, using only a fraction of data for training.~\footnote{The generated QA pairs and the code can be found at \url{https://github.com/seanie12/Info-HCVAE}} 
\end{abstract}
\section{Introduction} 

\emph{Extractive Question Answering (QA)} is one of the most fundamental and important tasks for natural language understanding. Thanks to the increased complexity of deep neural networks and use of knowledge transfer from the language models pre-trained on large-scale corpora~\citep{elmo, bert, unilm}, the state-of-the-art QA models have achieved human-level performance on several benchmark datasets \citep{squad1, squad2}. However, what is also crucial to the success of the recent data-driven models, is the availability of large-scale QA datasets. To deploy the state-of-the-art QA models to real-world applications, we need to construct high-quality datasets with large volumes of QA pairs to train them; however, this will be costly, requiring a massive amount of human efforts and time. 

\begin{table}
	\small
	\centering
	\begin{tabular}{lc}
		\midrule[0.5pt]
		\textbf{Paragraph (Input) } \hlpink{Philadelphia} has more murals than \\
		any other u.s. city, thanks in part to \hlyellow{the} \hlcyan{1984} \hlyellow{creation} \\
		\hlyellow{of the department of recreation's mural} \hlyellow{arts program},\\
		\ldots The program has funded \hlgreen{more than} \hlgreen{2,800 murals}  \\
		\midrule[0.5pt]
		\textbf{Q1}\: which city has more murals than any other city? \\
		\textbf{A1}\: \hlpink{philadelphia} \\
		\midrule[0.5pt]
		\textbf{Q2}\: why philadelphia has more murals? \\
		\textbf{A2}\: \hlyellow{the 1984 creation of the department of recreation's} \\
		\:\:\:\:\:\:\:\hspace{0.04cm}\hlyellow{mural arts program} \\
		\midrule[0.5pt]
		\textbf{Q3}\: when did the department of recreation' s mural \\
		\:\:\:\:\:\:\: arts program start ? \\
		\textbf{A3}\: \hlcyan{1984} \\
		\midrule[0.5pt]
		\textbf{Q4}\: how many murals funded the graffiti arts program \\
		\:\:\:\:\:\: by the department of recreation? \\
		\textbf{A4}\: \hlgreen{more than 2,800} \\
		\midrule[0.5pt]
	\end{tabular}
	\captionsetup{font=small}
	\caption{An example of QA pairs generated with our framework. The paragraph is an extract from Wikipedia provided by~\citet{harvesting}. For more examples, please see Appendix \ref{qualitativeexamples}.}\label{example}
\end{table}

\emph{Question generation (QG)}, or \emph{Question-Answer pair generation (QAG)}, is a popular approach to overcome this data scarcity challenge. Some of the recent works resort to semi-supervised learning, by leveraging large amount of unlabeled text (e.g. Wikipedia) to generate synthetic QA pairs with the help of QG systems \cite{qaqgdualtask, gdan, collaborative, self-training}. However, existing QG systems have overlooked an important point that generating QA pairs from a context consisting of unstructured texts, is essentially a \textit{one-to-many} problem. Sequence-to-sequence models are known to generate generic sequences~\citep{discourse} without much variety, as they are trained with maximum likelihood estimation.
This is highly suboptimal for QAG since the contexts given to the model often contain richer information that could be exploited to generate multiple QA pairs.  

To tackle the above issue, we propose a novel probabilistic deep generative model for QA pair generation. Specifically, our model is a hierarchical conditional variational autoencoder (HCVAE) with two separate latent spaces for question and answer conditioned on the context, where the answer latent space is additionally conditioned on the question latent space. During generation, this hierarchical conditional VAE first generates an answer given a context, and then generates a question given both the answer and the context, by sampling from both latent spaces. This probabilistic approach allows the model to generate diverse QA pairs focusing on different parts of a context at each time. 

Another crucial challenge of the QG task is to ensure the \emph{consistency} between a question and its corresponding answer, since they should be semantically dependent on each other such that the question is answerable from the given answer and the context. In this paper, we tackle this consistency issue by maximizing the mutual information \cite{mine, deepinfomax, qainfomax} between the generated QA pairs. We empirically validate that the proposed mutual information maximization significantly improves the QA-pair consistency. Combining both the hierarchical CVAE and the InfoMax regularizer together, we propose a novel probabilistic generative QAG model which we refer to as \textbf{Info}rmation Maximizing \textbf{H}ierarchical \textbf{C}onditional \textbf{V}ariational \textbf{A}uto\textbf{E}ncoder (\textbf{Info-HCVAE}). Our Info-HCVAE generates diverse and consistent QA pairs even from a very short context (see Table~\ref{example}). 


But how should we quantitatively measure the quality of the generated QA pairs? Popular evaluation metrics (e.g. BLEU \cite{bleu}, ROUGE \cite{rouge}, METEOR \cite{meteor}) for text generation only tell how similar the generated QA pairs are to Ground-Truth (GT) QA pairs, and are not directly correlated with their actual quality~\cite{qbleu, semantic}. Therefore, we use the {\bf QA}-based {\bf E}valuation ({\bf QAE}) metric proposed by~\citet{semantic}, which measures how well the generated QA pairs match the distribution of GT QA pairs. Yet, in a semi-supervised learning setting where we already have GT labels, we need novel QA pairs that are different from GT QA pairs for the additional QA pairs to be truly effective. Thus, we propose a novel metric, {\bf R}everse {\bf QAE} ({\bf R-QAE}), which is low if the generated QA pairs are novel and diverse. 

We experimentally validate our QAG model on SQuAD v1.1 \cite{squad1}, Natural Questions \cite{naturalquestion}, and TriviaQA \cite{triviaqa} datasets, with both QAE and R-QAE using BERT-base \cite{bert} as the QA model. Our QAG model obtains high QAE and low R-QAE, largely outperforming state-of-the-art baselines using a significantly smaller number of contexts. Further experimental results for semi-supervised QA on the three datasets using the SQuAD as the labeled dataset show that our model achieves significant improvements over the state-of-the-art baseline (+2.12 on SQuAD, +5.67 on NQ, and +1.18 on Trivia QA in EM).

Our contribution is threefold:
\begin{itemize}[itemsep=0mm, parsep=0pt, leftmargin=*]
\item We propose a novel hierarchical variational framework for generating diverse QA pairs from a single context, which is, to our knowledge, the first probabilistic generative model for question-answer pair generation (QAG).  
	
\item We propose an InfoMax regularizer which effectively enforces the consistency between the generated QA pairs, by maximizing their mutual information. This is a novel approach in resolving consistency between QA pairs for QAG.  
	
\item We evaluate our framework on several benchmark datasets by either training a new model entirely using generated QA pairs (QA-based evaluation), or use both ground-truth and generated QA pairs (semi-supervised QA). Our model achieves impressive performances on both tasks, largely outperforming existing QAG baselines.
\end{itemize}
\vspace{-0.1in}
\noindent

\vspace{-0.1in}
\section{Related Work}
\noindent
\textbf{Question and Question-Answer Pair Generation} Early works on Question Generation (QG) mostly resort to rule-based approaches \cite{earlyqg1,earlyqg2,earlyqg3}. However, recently, encoder-decoder based neural architectures~\cite{learningtoask,nqg} have gained popularity as they outperform rule-based methods. Some of them use paragraph-level information~\cite{harvesting, addinfo1, addinfo2,maxoutpointer, separation, answerfocused} as additional information. Reinforcement learning is a popular approach to train the neural QG models, where the reward is defined as the evaluation metrics~\cite{metricasreward1,metricasreward2}, or the QA accuracy/likelihood~\cite{qaaccu1,qaaccu2,semantic}. State-of-the-art QG models~\cite{googlegod, unilm, recurrentbertqg} use pre-trained language models. Question-Answer Pair Generation (QAG) from contexts, which is our main target, is a relatively less explored topic tackled by only a few recent works~\cite{harvesting, googlegod, unilm}. To the best of our knowledge, we are the first to propose a probabilistic generative model for end-to-end QAG; \citet{teachingmachinetoask} use VAE for QG, but they do not tackle QAG. Moreover, we effectively resolve the QA-pair consistency issue by maximizing their mutual information with an InfoMax regularizer~\cite{mine, deepinfomax, qainfomax}, which is another contribution of our work.

\noindent
\textbf{Semi-supervised QA with QG}
With the help of QG models, it is possible to train the QA models in a semi-supervised learning manner to obtain improved performance. \citet{qaqgdualtask} apply dual learning to jointly train QA and QG on unlabeled dataset. \citet{gdan} and \citet{collaborative} train QG and QA in a GAN framework~\citep{gan}. \citet{self-training} propose a curriculum learning to supervise the QG model to gradually generate difficult questions for the QA model. \citet{simple} introduce a cloze-style QAG method to pretrain a QA model. \citet{semantic} propose to filter out low-quality synthetic questions by the answer likelihood. While we focus on the answerable setting in this paper, few recent works tackle the unanswerable settings. \citet{unanswerable} use neural networks to edit answerable questions into unanswerable ones, and perform semi-supervised QA. \citet{googlegod} and \citet{unilm} convert generated questions into unanswerable ones using heuristics, and filter or replace corresponding answers based on EM or F1.

\noindent
\textbf{Variational Autoencoders} Variational autoencoders (VAEs)~\cite{vae} are probabilistic generative models used in a variety of natural language understanding tasks, including language modeling \cite{bowman}, dialogue generation \cite{hvae, cvae-dialogue, hvae2, dialogue-multiple-latent, dialogue-multiple}, and machine translation \cite{mt-vae, mt-rnn-vae, var-attention}. In this work, we propose a novel hierarchical conditional VAE framework with an InfoMax regularization for generating a pair of samples with high consistency.

\section{Method}

Our goal is to generate diverse and consistent QA pairs to tackle the data scarcity challenge in the extractive QA task. Formally, given a \emph{context} $\mathbf{c}$ which contains $M$ tokens, $\mathbf{c}=(c_1, \ldots, c_{M})$, we want to generate QA pairs $(\mathbf{x,y})$ where $\mathbf{x} = (x_1, \ldots, x_{N} )$ is the question containing $N$ tokens and $\mathbf{y} = (y_1, \ldots, y_{L})$ is its corresponding answer containing $L$ tokens. We aim to tackle the QAG task by learning the conditional joint distribution of the question and answer given the context, $p(\mathbf{x,y}|\mathbf{c})$, from which we can sample the QA pairs:
\begin{align*}
    \mathbf{(x, y)} \sim p(\mathbf{x, y}| \mathbf{c})
\end{align*}
We estimate $p(\mathbf{x,y}|\mathbf{c})$ with a probabilistic deep generative model, which we describe next.

\begin{figure*}
	\begin{center}
		\includegraphics[width=1.0\linewidth]{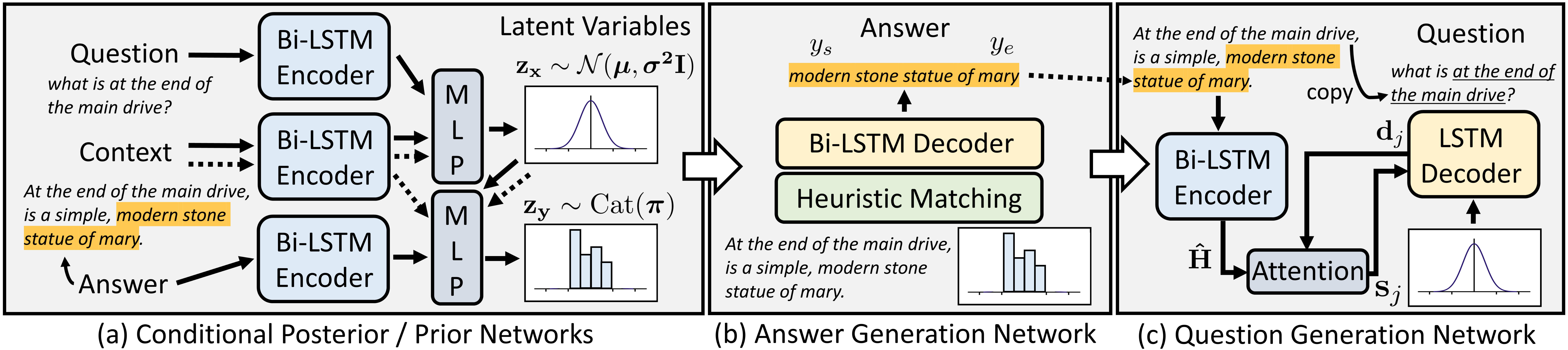}
	\end{center}
	\captionsetup{font=small}
	\caption{The conceptual illustration of the proposed HCVAE model encoding and decoding question and its corresponding answer jointly. The dashed line refers to the generative process of HCVAE.}\label{fig1}
	\vspace{-0.1in}
\end{figure*}
\subsection{Hierarchical Conditional VAE}
We propose to approximate the unknown conditional joint distribution $p(\mathbf{x, y}| \mathbf{c})$, with a variational autoencoder (VAE) framework~\cite{vae}. However, instead of directly learning a common latent space for both question and answer, we model $p(\mathbf{x}, \mathbf{y}|\mathbf{c})$ in a hierarchical conditional VAE framework with a separate latent space for question and answer as follows:
\begin{align*}
\begin{split}
&p_\theta({\bf x}, {\bf y}|{\bf c}) \\
&= \int_{{\bf z_x}} \sum_{{\bf z_y}} p_\theta({\bf x}|{\bf z_x}, {\bf y}, {\bf c}) p_\theta({\bf y} |\rvzx, {\bf z_y}, {\bf c}) \cdot\\
&\quad\quad\quad\quad\quad p_\psi({\bf z_y} | {\bf z_x, c}) p_\psi({\bf z_x} | {\bf c}) d{\bf z_x}
\end{split}
\end{align*}
where $\rvzx$ and $\rvzy$ are latent variables for question and answer respectively, and the $p_{\psi}({\bf z_x}|{\bf c})$ and $p_{\psi}({\bf z_y}|{\bf z_x, c})$ are their conditional priors following an isotropic Gaussian distribution and a categorical distribution (Figure \ref{fig1}-(a)). We decompose the latent space of question and answer, since the answer is always a finite span of context $\mathbf{c}$, which can be modeled well by a categorical distribution, while a continuous latent space is a more appropriate choice for question since there could be unlimited valid questions from a single context. Moreover, we design the bi-directional dependency flow of joint distribution for QA. By leveraging hierarchical structure, we enforce the answer latent variables to be dependent on the question latent variables in $p_\psi({\bf z_y} | {\bf z_x, c})$ and achieve the reverse dependency by sampling question $\mathbf{x} \sim p_\theta({\bf x}|{\bf z_x}, {\bf y}, {\bf c})$. We then use a variational posterior $q_{\phi}(\cdot)$ to maximize the Evidence Lower Bound (ELBO) as follows (The complete derivation is provided in \textbf{Appendix} \ref{derivation}):
\begin{align*}
\begin{split}
\log p_\theta(\mathbf{x}, \mathbf{y} | \mathbf{c}) &\geq \E_{\rvzx \sim q_{\phi}({\bf z_x}|{\bf x, c})} [\log{ p_{\theta}({\bf x}|{\bf z_x}, {\bf y}, {\bf c})}] \\
&\hspace{-0.2cm}+ \E_{{\bf z_y} \sim q_{\phi}({\bf z_y}|{\bf z_x, y, c})} [\log{ p_{\theta}({\bf y}|{\bf z_y}, {\bf c})}] \\
&\hspace{-0.2cm}- \kl [q_{\phi}({\bf z_y}|{\bf z_x, y, c})||p_{\psi}({\bf z_y}|{\bf z_x}, {\bf c})]  \\ 
&\hspace{-0.2cm}- \text{D}_{\text{KL}}[q_{\phi}({\bf z_x}|{\bf x, c})||p_{\psi}({\bf z_x}|{\bf c})] \\
&\hspace{-0.2cm} \equalscolon \loss_{\text{HCVAE}}
\end{split}
\end{align*}
where $\theta$, $\phi$, and $\psi$ are the parameters of the generation, posterior, and prior network, respectively. We refer to this model as a \emph{Hierarchical Conditional Variational Autoencoder (HCVAE)} framework.

\begin{figure}[t]
	\begin{center}
		\includegraphics[width=0.7\linewidth]{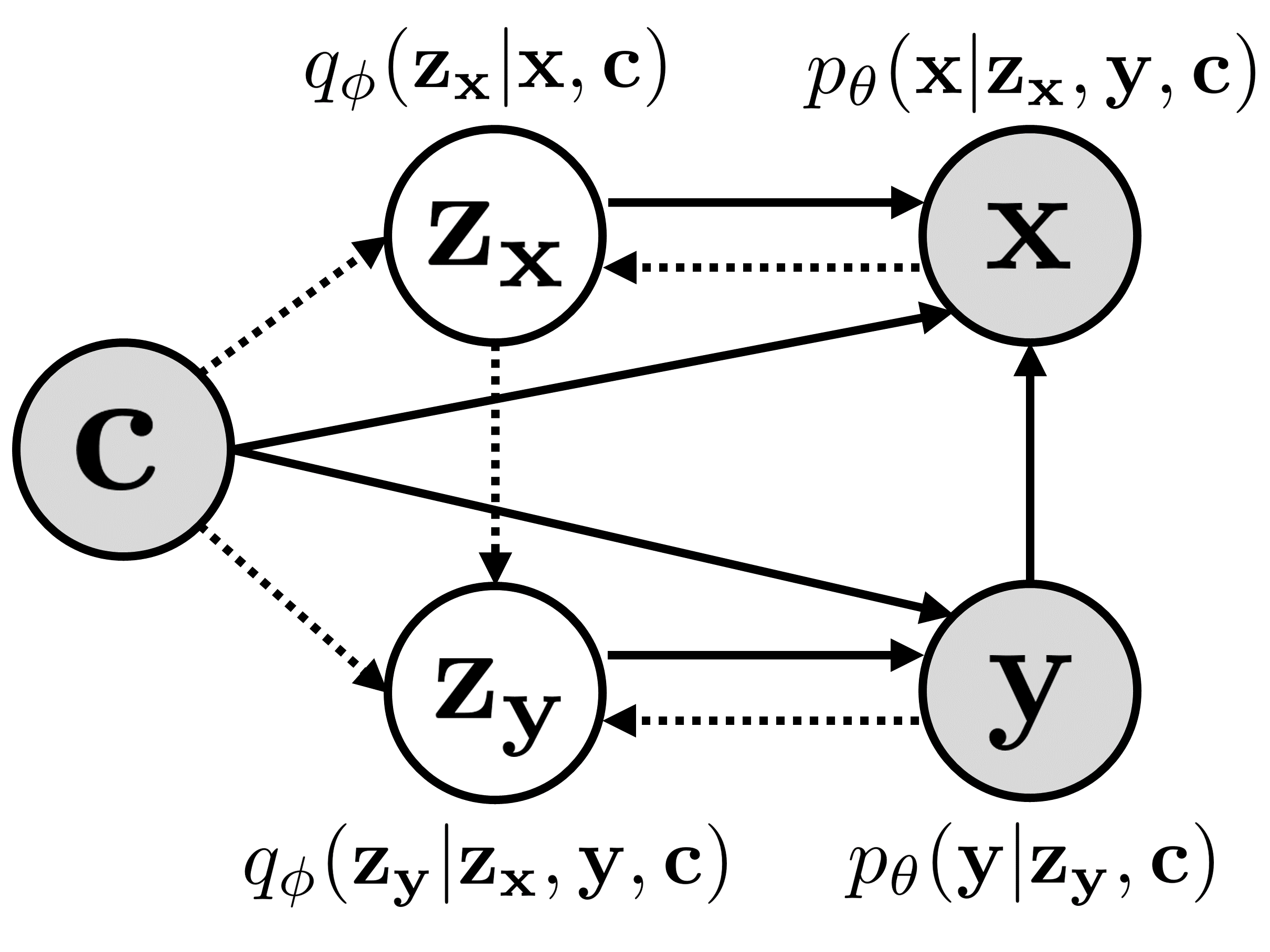}
	\end{center}
	\captionsetup{font=small}
	\caption{The directed graphical model for HCVAE. The gray and white nodes denote observed and latent variables.}\label{graphic}
\end{figure}

Figure~\ref{graphic} shows the directed graphical model of our HCVAE. The generative process is as follows:
\begin{enumerate}[itemsep=1mm, parsep=0pt, leftmargin=*]
	\item {Sample question L.V.: $\rvzx \sim p_{\psi}(\rvzx|\rvc)$}
	\item  {Sample answer L.V.: $\rvzy \sim p_{\psi}(\rvzy|\rvzx, \rvc)$}
	\item {Generate an answer: $\rvy \sim p_{\theta}(\rvy|\rvzy, \rvc)$}
	\item {Generate a question: $\rvx \sim p_{\theta}(\rvx|\rvzx, \rvy, \rvc)$}
\end{enumerate}

\noindent
\textbf{Embedding} We use the pre-trained word embedding network from BERT \cite{bert} for posterior and prior networks, whereas the whole BERT is used as a contextualized word embedding model for the generative networks. For the answer encoding, we use a binary token type id of BERT. Specifically, we encode all context tokens as 0s, except for the tokens which are part of answer span (highlighted words of context in Figure \ref{fig1}-(a) or -(c)), which we encode as 1s. We then feed the sequence of the word token ids, token type ids, and position ids into the embedding layer to encode the answer-aware context. We fix all the embedding layers in HCVAE during training.

\noindent
\textbf{Prior Networks} We use two different conditional prior networks $p_{\psi}(\mathbf{z_x}|\mathbf{c}), p_{\psi}(\mathbf{z_y}|\mathbf{z_x}, \mathbf{c})$ to model context-dependent priors (the dashed lines in Figure ~\ref{fig1}-(a)). To obtain the parameters of isotropic Gaussian $\mathcal{N}(\bm{\mu}, \bm{\sigma}^2\mathbf{I})$ for $p_{\psi}(\mathbf{z_x}|\mathbf{c})$, we use a bi-directional LSTM (Bi-LSTM) to encode the word embeddings of the context into the hidden representations, and then feed them into a Multi-Layer Perceptron (MLP). We model $p_{\psi}(\mathbf{z_y}|\mathbf{z_x},\mathbf{c})$ following a categorical distribution $\text{Cat}({\bm{\pi}})$, by computing the parameter $\bm{\pi}$ from $\mathbf{z_x}$ and the hidden representation of the context using another MLP.

\noindent
\textbf{Posterior Networks} We use two conditional posterior networks $q_{\phi}(\mathbf{z_x}|\mathbf{x}, \mathbf{c}), q_{\phi}(\mathbf{z_y}|\mathbf{z_x}, \mathbf{y},\mathbf{c})$ to approximate true posterior distributions of latent variables for both question $\mathbf{x}$ and answer $\mathbf{y}$. We use two Bi-LSTM encoders to output the hidden representations of question and context given their word embeddings. Then, we feed the two hidden representations into MLP to obtain the parameters of Gaussian distribution, $\bm{\mu'}$ and $\bm{\sigma'}$ (upper right corner in Figure \ref{fig1}-(a)). We use the reparameterization trick~\citep{vae} to train the model with backpropagation since the stochastic sampling process $\rvzx \sim q_{\phi}(\mathbf{z_x}|\mathbf{x},\mathbf{c})$ is nondifferentiable. We use another Bi-LSTM to encode the word embedding of answer-aware context into the hidden representation. Then, we feed the hidden representation and $\rvzx$ into MLP to compute the parameters $\bm{\pi}'$ of categorical distribution (lower right corner in Figure~\ref{fig1}-(a)). We use the categorical reparameterization trick with gumbel-softmax \citep{gumbelsoftmax2, gumbelsoftmax1} to enable backpropagation through sampled discrete latent variables. 

\noindent
\textbf{Answer Generation Networks} Since we consider extractive QA, we can factorize $p_{\theta}(\mathbf{y|z_y, c})$ into $p_{\theta}(y_s| \mathbf{z_y, c})$ and $p_{\theta}(y_e| \mathbf{z_y, c})$, where $y_s$ and $y_e$ are the start and the end position of an answer span (highlighted words in Figure \ref{fig1}-(b)), respectively. To obtain MLE estimators for both, we first encode the context $\rvc$ into the contextualized word embedding of $\mathbf{E^c}=\{ \mathbf{e}^{\bf c}_1, \ldots, \mathbf{e}^{\bf c}_M \}$ with the pre-trained BERT. We compute the final hidden representation of context and the latent variable $\rvzy$ with a heuristic matching layer \citep{heuristicmatching} and a Bi-LSTM:
\begin{gather*}
\mathbf{f}_i = [\mathbf{e}^{\bf c}_i;\: \rvzy ;\: |\mathbf{e}^{\bf c}_i - \rvzy|;\: \mathbf{e}^{\bf c}_i \odot \rvzy]\\
\overrightarrow{\mathbf{h}}_{i} = \overrightarrow{\text{LSTM}}([\mathbf{f}_i, \overrightarrow{\mathbf{h}}_{i-1}]) \\
\overleftarrow{\mathbf{h}}_{i} = \overleftarrow{\text{LSTM}}([\mathbf{f}_i, \overleftarrow{\mathbf{h}}_{i+1}]) \\
\mathbf{H} = [\:\overrightarrow{\mathbf{h}}_{i};\: \overleftarrow{\mathbf{h}}_{i}\:]^{M}_{i=1}
\end{gather*}
where $\rvzy$ is linearly transformed, and $\mathbf{H} \in \mathbb{R}^{d_{\rvy} \times M}$ is the final hidden representation. Then, we feed $\mathbf{H}$ into two separate linear layers to predict $y_s$ and $y_e$.

\noindent
\textbf{Question Generation Networks} 
We design the encoder-decoder architecture for our QG network by mainly adopting from our baselines \cite{maxoutpointer, semantic}. For encoding, we use pre-trained BERT to encode the answer-specific context into the contextualized word embedding, and then use a two-layer Bi-LSTM to encode it into the hidden representation (in Figure~\ref{fig1}-(c)). We apply a gated self-attention mechanism \cite{gated} to the hidden representation to better capture long-term dependencies within the context, to obtain a new hidden representation $\mathbf{\hat{H}} \in \mathbb{R}^{d_{\rvx} \times M}$. 

The decoder is a two-layered LSTM which receives the latent variable $\rvzx$ as an initial state. It uses an attention mechanism~\cite{luong-attention} to dynamically aggregate $\mathbf{\hat{H}}$ at each decoding step into a context vector of $\mathbf{s}_j$, using the $j$-th decoder hidden representation $\mathbf{d}_j \in \mathbb{R}^{d_{\rvx}}$ (in Figure \ref{fig1}-(c)). Then, we feed $\mathbf{d}_j$ and $\mathbf{s}_j$ into MLP with maxout activation \cite{maxout} to compute the final hidden representation $\mathbf{\hat{d}}_j$ as follows:
\begin{gather*}
\mathbf{d}_{0} = \mathbf{z_x}, \:\: \mathbf{d}_{j} = \text{LSTM}([\mathbf{e}^{\bf x}_{j-1}, \mathbf{d}_{j-1}]) \\
\mathbf{r}_j =  \mathbf{\hat{H}}^T\mathbf{W^a}\mathbf{d}_{j} , \:\: \mathbf{a}_j = \text{softmax}(\mathbf{r}_j), \:\: \mathbf{s}_j = \mathbf{\hat{H}} \mathbf{a}_j \\
\mathbf{\hat{d}}_{j} = \text{MLP}([\:\mathbf{d}_{j};\:\mathbf{s}_j\:])
\end{gather*}
where $\rvzx$ is linearly transformed, and $\mathbf{e}^{\bf x}_{j}$ is the $j$-th question word embedding. The probability vector over the vocabulary is computed as $p(\mathbf{x}_j|\rvx_{<j}, \rvzx, \rvy, \rvc) = \text{softmax}(\mathbf{W^e} \mathbf{\hat{d}}_{j})$. We initialize the weight matrix $\mathbf{W^e}$ as the pretrained word embedding matrix and fix it during training. Further, we use the copy mechanism \cite{maxoutpointer}, so that the model can directly copy tokens from the context. We also greedily decode questions to ensure that all stochasticity comes from the sampling of the latent variables.


\subsection{Consistent QA Pair Generation with Mutual Information Maximization}
One of the most important challenges of the QAG task is enforcing consistency between the generated question and its corresponding answer. They should be semantically consistent, such that it is possible to predict the answer given the question and the context. However, neural QG or QAG models often generate questions irrelevant to the context and the answer~\cite{semantic} due to the lack of the mechanism enforcing this consistency. We tackle this issue by maximizing the mutual information (MI) of a generated QA pair, assuming that an answerable QA pair will have high MI. Since an exact computation of MI is intractable, we use a neural approximation. While there exist many different approximations~\cite{mine,deepinfomax}, we use the estimation proposed by \citet{qainfomax} based on Jensen-Shannon Divergence:
\begin{align*}
\begin{split}
    \text{MI}(X; Y) &\geq \mathbb{E}_{\rvx, \rvy \sim \mathbb{P}}[\log g(\rvx, \rvy)] \\
    &  +\frac{1}{2} \mathbb{E}_{\tilde{\rvx}, \rvy \sim \mathbb{N}}[\log(1 - g(\tilde{\rvx}, \rvy))] \\ 
    & + \frac{1}{2} \mathbb{E}_{\rvx, \tilde{\rvy} \sim \mathbb{N}}[\log(1 - g(\rvx, \tilde{\rvy}))]\\
    & \equalscolon \mathcal{L}_{\text{Info}}
\end{split}
\end{align*}
where $\mathbb{E_P}$ and $\mathbb{E_N}$ denote expectation over positive and negative examples. We generate negative examples by shuffling the QA pairs in the minibatch, such that a question is randomly associated with an answer. Intuitively, the function $g(\cdot)$ acts like a binary classifier that discriminates whether QA pair is from joint distribution or not. We empirically find that the following $g(\cdot)$ effectively achieves our goal of consistent QAG:
\begin{gather*}
    g(\rvx, \rvy) = \text{sigmoid}(\mathbf{\overline{x}}^{T}\mathbf{W}\mathbf{\overline{y}})
\end{gather*}
where $\mathbf{\overline{x}}=\frac{1}{N}\sum_i \mathbf{\hat{d}}_{i}$ and $\mathbf{\overline{y}}= \frac{1}{L}\sum_j \mathbf{\hat{h}}_{j}$ are summarized representations of question and answer, respectively. Combined with the ELBO, the final objective of our Info-HCVAE is as follows:
\begin{align*}
\max_{\Theta} \: \loss_{\text{HCVAE}}\: +\: \lambda \mathcal{L}_{\text{Info}}
\vspace{-0.1in}
\end{align*}
where $\Theta$ includes all the parameters of $\phi, \psi, \theta$ and $\bf W$, and $\lambda$ controls the effect of MI maximization. In all experiments, we always set the $\lambda$ as 1.

\section{Experiment}

\subsection{Dataset}

\noindent
\textbf{Stanford Question Answering Dataset v1.1 (SQuAD)}~\citep{squad1}. This is a reading comprehension dataset consisting of questions obtained from crowdsourcing on a set of Wikipedia articles, where the answer to every question is a segment of text or a span from the corresponding reading passage. We use the same split used in~\citet{semantic} for the fair comparison.

\noindent
\textbf{Natural Questions (NQ)}~\citep{naturalquestion}. This dataset contains realistic questions from actual user queries to a search engine, using Wikipedia articles as context. We adapt the dataset provided from MRQA shared task~\cite{mrqa} and convert it into the extractive QA format. We split the original validation set in half, to use as validation and test for our experiments.

\noindent
\textbf{TriviaQA}~\citep{triviaqa}. This is a reading comprehension dataset containing question-answer-evidence triples. The QA pairs and the evidence (contexts) documents are authored and uploaded by Trivia enthusiasts. Again, we only choose QA pairs of which answers are span of contexts.

\noindent
\textbf{HarvestingQA}~\footnote{\url{https://github.com/xinyadu/harvestingQA}} This dataset contains top-ranking 10K Wikipedia articles and 1M synthetic QA pairs generated from them, by the answer span extraction and QG system proposed in~\citep{harvesting}. We use this dataset for semi-supervised learning.

\subsection{Experimental Setups}
\noindent {\bf Implementation Details}
In all experiments, we use BERT-base ($d = 768$) \citep{bert} as the QA model, setting most of the hyperparameters as described in the original paper. For both HCVAE and Info-HCVAE, we set the hidden dimensionality of the Bi-LSTM to $300$ for posterior, prior, and answer generation networks, and use the dimensionality of $450$ and $900$ for the encoder and the decoder of the question generation network. We set the dimensionality of $\rvzx$ as $50$, and define $\rvzy$ to be set of 10-way categorical variables $\rvzy=\{\mathbf{z}_1, \ldots, \mathbf{z}_{20}\}$. For training the QA model, we fine-tune the model for 2 epochs. We train both the QA model and Info-HCVAE with Adam optimizer \citep{adam} with the batch size of $32$ and the initial learning rate of $5\cdot10^{-5}$ and $10^{-3}$ respectively. For semi-supervised learning, we first pre-train BERT on the synthetic data for 2 epochs and fine-tune it on the GT dataset for 2 epochs. To prevent \textit{posterior collapse}, we multiply $0.1$ to the KL divergence terms of question and answer~\cite{betavae}. For more details of the datasets and experimental setup, please see \textbf{Appendix} \ref{training-detail}.

\noindent
\textbf{Baselines} We experiment two variants of our model against several baselines:
\begin{enumerate}[itemsep=0mm, parsep=0pt, leftmargin=*]
    \item \textbf{Harvest-QG}: An attention-based neural QG model with a neural answer extraction system \cite{harvesting}.
    \item \textbf{Maxout-QG}: A neural QG model based on maxout copy mechanism with a gated self-attetion \cite{maxoutpointer}, which uses BERT as the word embedding as suggested by~\citet{semantic}.
    \item \textbf{Semantic-QG}:  A neural QG model based on Maxout-QG with semantic-enhanced reinforcement learning~\cite{semantic}.
    \item \textbf{HCVAE}:  Our HCVAE model without the InfoMax regularizer.
    \item \textbf{Info-HCVAE}: Our full model with the InfoMax regularizer.
\end{enumerate}

\noindent
For the baselines, we use the same answer spans extracted by the answer extraction system \cite{harvesting}.

\subsection{Quantitative Analysis}

\noindent
\textbf{QAE and R-QAE} One of crucial challenges with generative models is a lack of a good quantitative evaluation metric. We adopt \textbf{QA}-based \textbf{E}valuation (QAE) metric proposed by \citet{semantic} to measure the quality of QA pair. QAE is obtained by first training the QA model on the synthetic data, and then evaluating the QA model with human annotated test data. However, QAE only measures how well the distribution of synthetic QA pairs matches the distribution of GT QA pairs, and does not consider the diversity of QA pairs. Thus, we propose {\bf R}everse {\bf QA}-based {\bf E}valuation ({\bf R-QAE}), which is the accuracy of the QA model trained on the human-annotated QA pairs, evaluated on the generated QA pairs. If the synthetic data covers larger distribution than the human annotated training data, R-QAE will be lower. However, note that having a low R-QAE is only meaningful when the QAE is high enough since trivially invalid questions may also yield low R-QAE.

\begin{table}
	\small
	\centering
	\begin{tabular}{lcc}
		\midrule[0.8pt]
		{\textbf{Method}} & {\textbf{QAE ($\mathbf{\uparrow}$)}} & {\textbf{R-QAE ($\mathbf{\downarrow}$)}} \\
		\midrule[0.8pt]
		\multicolumn{3}{c}{\bf SQuAD (EM/F1)} \\ 
		\midrule[0.8pt]
		{Harvesting-QG} & {55.11/66.40} & {64.77/78.85}\\
		{Maxout-QG} & {56.08/67.50} & {62.49/78.24} \\
		{Semantic-QG} & {60.49/71.81} & {74.23/88.54}\\
		\midrule[0.8pt]
		{\bf HCVAE} & {69.46/80.79} & {{\bf 37.57}/61.24}\\
		{\bf Info-HCVAE} & {\bf 71.18/81.51} & {38.80/{\bf 60.73}}\\
		\midrule[0.8pt]
		\multicolumn{3}{c}{\bf Natural Questions (EM/F1)} \\ 
		\midrule[0.8pt]
		{Harvesting-QG} & {27.91/41.23} & {49.89/70.01} \\
		{Maxout-QG} & {30.98/44.96} & {49.96/70.03} \\
		{Semantic-QG} & {30.59/45.29} & {58.42/79.23} \\
		\midrule[0.8pt]
		{\bf HCVAE} & {31.45/46.77} & {32.78/55.12} \\
		{\bf Info-HCVAE} & {\bf 37.18/51.46} & {\bf 29.39/53.04} \\
		\midrule[0.8pt]
		\multicolumn{3}{c}{\bf TriviaQA (EM/F1)} \\ 
		\midrule[0.8pt]
		{Harvesting-QG} & {21.32/30.21} & {29.75/47.73} \\
		{Maxout-QG} & {24.58/34.32} & {31.56/49.92} \\
		{Semantic-QG} & {27.54/38.25} & {37.45/58.15} \\
		\midrule[0.8pt]
		{\bf HCVAE} & {30.20/40.88} & {34.41/48.16} \\
		{\bf Info-HCVAE} & {\bf 35.45/44.11} & {\bf 21.65/37.65} \\
		\midrule[0.8pt]
	\end{tabular}
	\captionsetup{font=small}
	\caption{QAE and R-QAE results on three datasets. All results are the performances on our test set.}
	\label{qaeandrqaeonsquad}
\end{table}

\begin{table}
	\small
	\centering
	\begin{tabular}{ccc|cc}
		\midrule[0.8pt]
		{\bf Harvest} & {\bf Maxout} & {\bf Semantic} & \multirow{2}{*}{\bf HCVAE} & {\bf Info-}\\
		{\bf -QG} & {\bf -QG} & {\bf -QG} & {} & {\bf HCVAE}\\
		\midrule[0.8pt]
		{111.74} & {114.58} & {112.94} & {113.89} & {\bf 117.41} \\
		\midrule[0.8pt]
	\end{tabular}
	\captionsetup{font=small}
	\caption{The results of mutual information estimation. The results are based on QA pairs generated from H$\times$10\%.}
	\label{miresults}
\end{table}

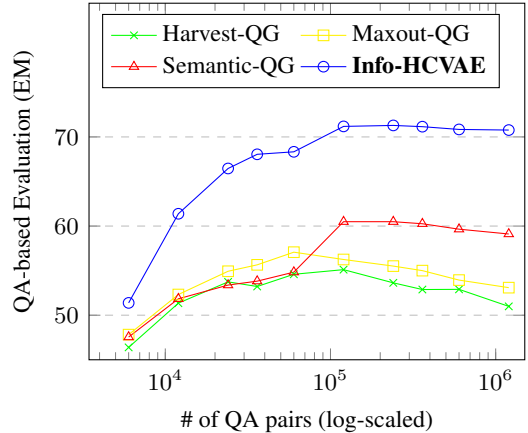
\begin{figure}[t]
\vspace{0.05in}
	\begin{center}
		\begin{tikzpicture}
		
		\begin{axis}[
		compat=1.10,
		xmode=log,
		xlabel={\# of QA pairs (log-scaled)},
		ylabel={QA-based Evaluation (EM)},
		legend cell align={left},
		legend columns=2
		xmin=1.0*1000, xmax=1.5*1000000,
		ymin=45, ymax=85,
		xtick={1000,10000,100000,1000000},
		ytick={50,60,70},
		legend pos=north west,
		ymajorgrids=true,
		grid style=dashed,
		legend style={font=\small},
		label style={font=\small},
		tick label style={font=\small}
		]
		\addplot[
		color=green,
		mark=x,
		]
		coordinates {
			(6*10^3,46.39)(1.2*10^4,51.35)(2.4*10^4,53.74)(3.6*10^4,53.22)(6*10^4,54.59)(1.2*10^5,55.11)(2.4*10^5,53.63)(3.6*10^5,52.88)(6*10^5,52.9)(1.2*10^6,50.99)
		};
		\addplot[
		color=yellow,
		mark=square,
		]
		coordinates {
			(6*10^3,47.82)(1.2*10^4,52.32)(2.4*10^4,54.93)(3.6*10^4,55.65)(6*10^4,57.07)(1.2*10^5,56.25)(2.4*10^5,55.51)(3.6*10^5,55.00)(6*10^5,53.94)(1.2*10^6,53.09)
		};
		\addplot[
		color=red,
		mark=triangle,
		]
		coordinates {
			(6*10^3,47.55)(1.2*10^4,51.84)(2.4*10^4,53.38)(3.6*10^4,53.82)(6*10^4,54.83)(1.2*10^5,60.49)(2.4*10^5,60.49)(3.6*10^5,60.26)(6*10^5,59.65)(1.2*10^6,59.11)
		};
		\addplot[
		color=blue,
		mark=o,
		]
		coordinates {
			(6*10^3,51.38)(1.2*10^4,61.38)(2.4*10^4,66.46)(3.6*10^4,68.05)(6*10^4,68.33)(1.2*10^5,71.18)(2.4*10^5,71.29)(3.6*10^5,71.15)(6*10^5,70.84)(1.2*10^6,70.77)
		};
		\legend{Harvest-QG, Maxout-QG, Semantic-QG, {\bf Info-HCVAE}}
		
		\end{axis}
		\end{tikzpicture}
	\end{center}
 	\caption{\small QAE vs. \# of QA pairs (log-scaled) on SQuAD.}
	\label{qaevsdata}
\end{figure}

\begin{table}[t]
	\small
	\centering
	\begin{tabular}{lcc}
		\midrule[0.8pt]
		{\textbf{Method}} & {\textbf{QAE ($\mathbf{\uparrow}$)}} & {\textbf{R-QAE ($\mathbf{\downarrow}$)}} \\
		\midrule[0.8pt]
		{Baseline} & {56.08/67.50} & {62.49/78.24} \\
		{+Q-latent} & {58.66/70.54} & {40.00/62.02} \\
		{+A-latent} & {69.46/80.79} & {\textbf{37.57}/61.24}\\
		{+InfoMax} & $\textbf{71.18/81.51}$ & 38.80/$\textbf{60.73}$\\
		\midrule[0.8pt]        
	\end{tabular}
	\captionsetup{font=small}
	\caption{QAE and R-QAE results of the ablation study on SQuAD dataset. All the results are the performances on our test set.}
	\label{ablationstudy}
\end{table}

\noindent
\textbf{Results} We compare HCVAE and Info-HCVAE with the baseline models on SQuAD, NQ, and TriviaQA. We use 10\% of Wikipedia paragraphs from HarvestingQA \cite{harvesting} for evaluation. Table~\ref{qaeandrqaeonsquad} shows that both HCVAE and Info-HCVAE significantly outperforms all baselines by large margin in QAE on all three datasets, while obtaining significantly lower R-QAE, which shows that our model generated both high-quality and diverse QA pairs from the given context. Moreover, Info-HCVAE largely outperforms HCVAE, which demonstrates the effectiveness of our InfoMax regularizer for enforcing QA-pair consistency. 

Figure~\ref{qaevsdata} shows the accuracy as a function of number of QA pairs. Our Info-HCVAE outperform all baselines by large margins using orders of magnitude smaller number of QA pairs. For example, Info-HCVAE achieves $61.38$ points using 12K QA pairs, outperforming Semantic-QG that use 10 times larger number of QA pairs.  We also report the score of $\mathbf{\overline{x}}^{T}\mathbf{W}\mathbf{\overline{y}}$ as an approximate estimation of mutual information (MI) between QA pairs generated by each method in Table~\ref{miresults}; our Info-HCVAE yields the largest value of MI estimation.

\noindent 
\textbf{Ablation Study} We further perform an ablation study to see the effect of each model component. We start with the model without any latent variables, which is essentially a deterministic Seq2Seq model (denoted as Baseline in Table~\ref{ablationstudy}). Then, we add in the question latent variable (+Q-latent) and then the answer latent variable (+A-latent), to see the effect of probabilistic latent variable modeling and hierarchical modeling respectively. The results in Table~\ref{ablationstudy} shows that both are essential for improving both the quality (QAE) and diversity (R-QAE) of the generated QA pairs. Finally, adding in the InfoMax regularization (+InfoMax) further improves the performance by enhancing the consistency of the generated QA pairs. 

\subsection{Qualitative Analysis}
\begin{table}[t]
	\small
	\centering
	\begin{tabular}{lccc}
		\midrule[0.8pt]
		{\textbf{Method}} & {\textbf{Diversity}} & {\textbf{Consistency}} & {\textbf{Overall}} \\
		\midrule[0.8pt]
		{Baseline} & {26\%} & {34\%} & {30\%}  \\ 
		{Ours} & {\bf 47\%} & {\bf 50\%} & {\bf 52\%}\\
		{Tie} & {27\%} & {16\%} & {18\%} \\
		\midrule[0.8pt]
	\end{tabular}
	\captionsetup{font=small}
	\caption{The results of human judgement in terms of diversity, consistency, and overall quality on the generated QA pairs.}
	\label{humanevaluation}
\end{table}

\begin{table}[t]
	\small
	\centering
	\begin{tabular}{l}
		\midrule[0.8pt]
		\textbf{Paragraph} 
		\hlpink{The scotland act 1998} which was passed by \\
		and given royal assent by queen Elizabeth ii on 19 \\
		november 1998, governs functions and role of the scottish \\ 
		parliament and delimits its legislative competence \ldots \\
		\midrule[0.8pt]
		\textbf{GT}\: what act sets forth the functions of the scottish \\
		\:\:\:\:\:\:\:\:\:parliament? \\
		\midrule[0.8pt]
		\textbf{O-1}\: which act was passed in 1998?\\
		\textbf{O-2}\: which act governs role of the scottish parliament?\\
		\textbf{O-3}\: which act was passed by queen Elizabeth ii? \\
		\textbf{O-4}\: which act gave the scottish parliament the \\ 
		\:\:\:\:\:\:\:\:\:responsibility to determine its legislative policy? \\
		\midrule[0.8pt]
	\end{tabular}
	\captionsetup{font=small}
	\caption{\small Examples of \textit{one-to-many} mapping of our Info-HCVAE. The answer is highlighted by pink. \textbf{GT} denotes the ground-truth question. \textbf{O-} denotes questions generated by Info-HCVAE.}
	\label{onetomany}
\end{table}

\def\limitarrow#1{%
	\begin{tikzpicture}
	\draw[<-] (0,1.75) to (0,0.5);
	\node at (0,0) {#1};
	\draw[<-] (0,-1.75) to (0,-0.5);
	\end{tikzpicture}}

\begin{table}[t]
	\small
	\centering
	\begin{tabular}{cl}
		\midrule[0.8pt]
		\multicolumn{2}{l}{\textbf{Paragraph} ... Atop the main building' s gold dome is} \\
		\multicolumn{2}{l}{a golden statue of the virgin mary. ... Next to the main}\\
		\multicolumn{2}{l}{building is \hlpink{the basilica of the sacred heart}. Immediately}\\
		\multicolumn{2}{l}{behind the basilica is the \hlyellow{grotto},  ... \hlcyan{a marian place of}} \\
		\multicolumn{2}{l}{\hlcyan{prayer and reflection}. ... At the end of the main drive ...,} \\
		\multicolumn{2}{l}{ is a simple, modern \hlgreen{stone statue of mary}.} \\
		\midrule[0.8pt]
		\multirow{2}{*}{\bf Ori1} & {{\bf Q} what is the grotto at notre dame?} \\ 
		&{{\bf A} a marian place of prayer and reflection} \\
		\midrule[0.8pt]
		\multirow{11}{*}{\limitarrow{\bf Gen}} & \textit{{\bf Q} where is the grotto at?} \\ 
		& \textit{{\bf A} \hlcyan{a marian place of prayer and reflection}} \\
		\cmidrule[0.8pt]{2-2}
		& \textit{{\bf Q} what place is behind the basilica of prayer?} \\
		& \textit{{\bf A} \hlyellow{grotto}}\\
		\cmidrule[0.8pt]{2-2}
		& \textit{{\bf Q} what is next to the main building at} \\
		& \textit{\:\:\:\:\:notre dame?} \\
		& \textit{{\bf A} \hlpink{the basilica of the sacred heart}} \\
		\cmidrule[0.8pt]{2-2}
		& \textit{{\bf Q} what is at the end of the main drive?}\\
		& \textit{{\bf A} \hlgreen{stone statue of mary}} \\
		\midrule[0.8pt]
		\multirow{3}{*}{\bf Ori2} & {{\bf Q} what sits on top of the main building at} \\
		& {\:\:\:\:\:notre dame?} \\
		& {{\bf A} a golden statue of the virgin mary}\\
		\midrule[0.8pt]
	\end{tabular}
	\captionsetup{font=small}
	\caption{QA pairs generated by interpolating between two latent codes encoded by our posterior networks. \textbf{Ori1} and \textbf{Ori2} are from training set of SQuAD.}
	\label{interporlation}
\end{table}

\noindent
\textbf{Human Evaluation} As a qualitative analysis, we first conduct a pairwise human evaluation of the QA pairs generated by our Info-HCVAE and Maxout-QG on 100 randomly selected paragraphs. Specifically, 20 human judges performed blind quality assessment of two sets of QA pairs that are presented in a random order, each of which contained two to five QA pairs. Each set of QA pairs is evaluated in terms of the overall quality, diversity, and consistency between the generated QA pairs and the context. The results in Table \ref{humanevaluation} show that the QA pairs generated by our Info-HCVAE is evaluated to be more diverse and consistent, compared to ones generated by the baseline models.

\noindent 
\textbf{One-to-Many QG} To show that our Info-HCVAE can effectively tackle \textit{one-to-many} mapping problem for question generation, we qualitatively analyze the generated questions for given a context and an answer from the SQuAD validation set. Specifically, we sample the question latent variables multiple times using the question prior network $p_{\psi}(\rvzx|\rvc)$, and then feed them to question generation networks $p_{\theta}(\rvx|\rvzx,\rvy,\rvc)$ with the answer. The example in Table \ref{onetomany} shows that our Info-HCVAE generates diverse and semantically consistent questions given an answer. We provide more qualitative examples in \textbf{Appendix} \ref{qualitativeexamples}. 

\noindent
\textbf{Latent Space Interpolation} To examine if Info-HCVAE learns meaningful latent space of QA pairs, we qualitatively analyze the QA pairs generated by interpolating between two latent codes of it on SQuAD training set. We first encode $\rvzx$ from two QA pairs using posterior networks of $q_{\phi}(\mathbf{z_x}|\mathbf{x,c})$, and then sample $\rvzy$ from interpolated values of $\rvzx$ using prior networks $p_{\psi}(\mathbf{z_y}|\mathbf{z_x, c})$ to generate corresponding QA pairs. Table \ref{interporlation} shows that the semantic of the QA pairs generated smoothly transit from one latent to another with high diversity and consistency. We provide more qualitative examples in \textbf{Appendix} \ref{qualitativeexamples}.

\subsection{Semi-supervised QA}
We now validate our model in a semi-supervised setting, where the model uses both the ground truth labels and the generated labels to solve the QA task, to see whether the generated QA pairs help improve the performance of a QA model in a conventional setting. Since such synthetic datasets consisting of generated QA pairs may inevitably contain some noise~\cite{semantic, unilm, googlegod}, we further refine the QA pairs by using the heuristic suggested by \citet{unilm}, to replace the generated answers whose F1 score to the prediction of the QA model trained on the human annotated data is lower than a set threshold. We select the threshold of $40.0$ for the QA pair refinement model via cross-validation on the SQuAD dataset, and used it for the experiments. Please see \textbf{Appendix} \ref{training-detail} for more details.

\noindent
\textbf{SQuAD} We first perform semi-supervised QA experiments on SQuAD using the synthetic QA pairs generated by our model. For the contexts, we use both the paragraphs in the original SQuAD (S) dataset, and the new paragraphs in the HarvestingQA dataset (H). Using Info-HCVAE, we generate 10 different QA pairs by sampling from the latent spaces (denoted as S$\times$10). For the baseline, we use Semantic-QG \cite{semantic} with the beam search size of $10$ to obtain the same number of QA pairs. We also generate new QA pairs using different portions of paragraphs provided in HarvestingQA (denoted as H$\times$10\%-H$\times$100\%), by sampling one latent variable per context. Table \ref{semionsquad} shows that our framework improves the accuracy of the BERT-base model by 2.12 (EM) and 1.59 (F1) points, significantly outperforming Semantic-QG.

\begin{table}[t]
	\small
	\centering
	\begin{tabular}{lll}
		\midrule[0.8pt]
		{\bf Data} & {\bf EM} & {\bf F1} \\
		\midrule[0.8pt]
		{SQuAD} & {80.25} & {88.23} \\
		\midrule[0.8pt]
		\multicolumn{3}{c}{{\bf Semantic-QG (baseline)}} \\
		\midrule[0.8pt]
		{+S$\times$10} & {81.20 (+0.95)} & {88.36 (+0.13)} \\
		{+H$\times$100\%} & {81.03 (+0.78)} & {88.79 (+0.56)} \\
		{+S$\times$10 + H$\times$100\%} & {81.44 (+1.19)} & {88.72 (+0.49)} \\
		\midrule[0.8pt]
		\multicolumn{3}{c}{\bf Info-HCVAE (ours)} \\
		\midrule[0.8pt]
		{+S$\times$10} & {82.09 (+1.84)} & {89.11 (+0.88)} \\
		{+H$\times$10\%} & {81.37 (+1.12)} & {88.85 (+0.62)} \\
		{+H$\times$20\%} & {81.68 (+1.43)} & {89.06 (+0.93)} \\
		{+H$\times$30\%} & {81.76 (+1.51)} & {89.12 (+0.89)} \\
		{+H$\times$50\%} & {82.17 (+1.92)} & {89.38 (+1.15)} \\
		{+H$\times$100\%} & {\bf 82.37 (+2.12)} & {89.63 (+1.40)} \\
		{+S$\times$10 + H$\times$100\%} & {82.19 (+1.94)} & {\bf 89.84 (+1.59)} \\
		\midrule[0.8pt]
	\end{tabular}
	\captionsetup{font=small}
	\caption{\small The results of semi-supervised QA experiments on SQuAD. All the results are the performances on our test set.}%
	\label{semionsquad}
\end{table}

\noindent
\textbf{NQ and TriviaQA} Our model is most useful when we do not have any labeled data for a target dataset. To show how well our QAG model performs in such a setting, we train the QA model using only the QA pairs generated by our model trained on SQuAD and test it on the target datasets (NQ and TriviaQA). We generate multiple QA pairs from each context of the target dataset, sampling from the latent space one to ten times (denoted by N$\times$1-10 or T$\times$1-10 in Table \ref{semionnqandtrivia}). Then, we fine-tune the QA model pretrained on the SQuAD dataset with the generated QA pairs from the two datasets. Table~\ref{semionnqandtrivia} shows that as we augment training data with larger number of synthetic QA pairs, the performance of the QA model significantly increases, significantly outperforming the QA model trained on SQuAD only. Yet, models trained with our QAG still largely underperform models trained with human labels, due to the distributional discrepancy between the source and the target dataset.

\section{Conclusion} We proposed a novel probabilistic generative framework for generating diverse and consistent question-answer (QA) pairs from given texts. Specifically, our model learns the joint distribution of question and answer given context with a hierarchically conditional variational autoencoder, while enforcing consistency between generated QA pairs by maximizing their mutual information with a novel InfoMax regularizer. To our knowledge, ours is the first successful probabilistic QAG model. We evaluated the QAG performance of our model by the accuracy of the BERT-base QA model trained using the generated questions on multiple datasets, on which it largely outperformed the state-of-the-art QAG baseline (+6.59-10.69 in EM), even with a smaller number of QA pairs. We further validated our model for semi-supervised QA, where it improved the performance of the BERT-base QA model on the SQuAD by 2.12 in EM, significantly outperforming the state-of-the-art model. As future work, we plan to extend our QAG model to a meta-learning framework, for generalization over diverse datasets.

\begin{table}[t]
	\small
	\centering
	\begin{tabular}{lll}
		\midrule[0.8pt]
		{\bf Data} & {\bf EM} & {\bf F1} \\
		\midrule[0.8pt]
		\multicolumn{3}{c}{{\bf Natural Questions}} \\
		\midrule[0.8pt]
		SQuAD & {42.77} & {57.29} \\
		\midrule[0.8pt]
		{+N$\times$1} & {46.70 (+3.94)} & {61.08 (+3.79)} \\
		{+N$\times$2} & {46.95 (+4.19)} & {61.34 (+4.05)} \\
		{+N$\times$3} & {47.73 (+4.96)} & {61.98 (+4.69)} \\
		{+N$\times$5} & {48.19 (+5.42)} & {62.21 (+4.92)} \\
		{+N$\times$10} & {\bf 48.44 (+5.67)} & {\bf 62.69 (+5.40)} \\
		\midrule[0.8pt]
		NQ & {61.65} & {73.91} \\
		\midrule[0.8pt]
		\multicolumn{3}{c}{{\bf TriviaQA}} \\
		\midrule[0.8pt]
		SQuAD & {48.96} & {57.98} \\
		\midrule[0.8pt]
		{+T$\times$1} & {49.65 (+0.69)} & {59.13 (+1.21)} \\
		{+T$\times$2} & {50.01 (+1.05)} & {59.08 (+1.10)} \\
		{+T$\times$3} & {49.71 (+0.75)} & {\bf 59.49 (+1.51)} \\
		{+T$\times$5} & {\bf 50.14 (+1.18)} & {59.21 (+1.23)} \\
		{+T$\times$10} & {49.65 (+0.69)} & {59.20 (+1.22)} \\
		\midrule[0.8pt]
		Trivia & {64.55} & {70.42} \\
		\midrule[0.8pt]
	\end{tabular}
	\captionsetup{font=small}
	\caption{The result of semi-supervised QA experiments on Natural Questions and TriviaQA dataset. All results are the performance on our test set.} %
	\label{semionnqandtrivia}
\end{table}

\vspace{-0.04in}
\section*{Acknowledgements}
\vspace{-0.1in}
This work was supported by the Engineering Research Center Program through the National Research Foundation of Korea (NRF) funded by the Korean Government MSIT (NRF-2018R1A5A1059921), Institute of Information \& communications Technology Planning \& Evaluation (IITP) grant funded by the Korea government(MSIT) (No.2019-0-01410, Research Development
of Question Generation for Deep Learning
based Semantic Search Domain Extension, 
No.2016-0-00563, Research
on Adaptive Machine Learning Technology
Development for Intelligent Autonomous Digital
Companion, No.2019- 0-00075, and Artificial Intelligence
Graduate School Program (KAIST)).
\bibliography{main}
\bibliographystyle{acl_natbib}

\clearpage
\appendix
\section*{Appendix}

\renewcommand{\thesubsection}{\Alph{subsection}}

\subsection{Derivation of Variational Lower Bound}
\label{derivation}
\begin{Thm*}
If we assume conditional independence of $\by$ and $\bz_\bx$, i.e., $p_{\theta}(\by | \bz_\bx, \bz_\by, \bc) = p_{\theta}(\by | \bz_\by, \bc)$, $\log p_\theta(\mathbf{x}, \mathbf{y} | \mathbf{c}) \geq \loss_{\text{HCVAE}}$
\end{Thm*}
\begin{proof}
\begin{align*}
    &\log p_{\theta}(\bx, \by| \bc)\\
    &= \log \int_{\bz_{\bx}} \sum_{\bz_{\by}} p_{\theta}(\bx| \bz_{\bx}, \by, \bc) \cdot \\
    &\quad p_{\theta}(\by|\bz_\bx, \bz_\by, \bc)p_{\psi}(\bz_\by|\bz_\bx, \bc)p_{\psi}(\bz_\bx|\bc) d_{\bz_\bx} \\
    &= \log \int_{\bz_\bx} p_{\theta}(\bx| \bz_\bx, \by, \bc) p_{\psi}(\bz_\bx| \bc) \frac{q_\phi(\bz_\bx|\bx, \bc)}{q_\phi(\bz_\bx|\bx, \bc)} \cdot \\
    &\quad\sum_{\bz_\by} p_{\theta}(\by|\bz_\by, \bc) p_{\psi}(\bz_\by|\bz_\bx, \bc) \frac{q_{\phi}(\bz_\by| \bz_\bx, \by, \bc)}{q_{\phi}(\bz_\by| \bz_\bx, \by, \bc)}  d_{\bz_\bx} \\
    &= \log \int_{\bz_\bx} p_{\theta}(\bx | \bz_\bx, \by, \bc) p_{\psi}(\bz_\bx | \bc) \frac{q_\phi(\bz_\bx|\bx, \bc)}{q_\phi(\bz_\bx|\bx, \bc)} \\
    &\quad\cdot \mathbb{E}_{q_{\phi}(\bz_\by | \bz_\bx, \by, \bc)}\Bigg[\frac{p_{\theta}(\by|\bz_\by, \bc)p_{\psi}(\bz_\by|\bz_\bx, \bc)}{q_{\phi}(\bz_\by| \bz_\bx, \by, \bc)} \Bigg] d_{\bz_\bx}\\
    &= \log \mathbb{E}_{q_{\phi}(\bz|\bx, \bc)}\{\frac{p_{\theta}(\bx| \bz_\bx, \by, \bc)p_{\psi}(\bz_\bx|\bc)}{q_{\phi}(\bz_\bx| \bx, \bc)} \cdot \\
    &\quad \mathbb{E}_{q_{\phi}(\bz_\by |\bz_\bx, \by, \bc)}\Bigg[\frac{p_{\theta}(\by|\bz_\by, \bc)p_{\psi}(\bz_\by|\bz_\bx, \bc)}{q_{\phi}(\bz_\by| \bz_\bx, \by, \bc)} \Bigg]\} \\
    &\geq \mathbb{E}_{q_{\phi}(\bz|\bx, \bc)}\{\log \frac{p_{\theta}(\bx| \bz_\bx, \by, \bc)p_{\psi}(\bz_\bx|\bc)}{q_{\phi}(\bz_\bx| \bx, \bc)} + \\
    &\quad\log\mathbb{E}_{q_{\phi}(\bz_\by |\bz_\bx, \by , \bc)}\Bigg[\frac{p_{\theta}(\by|\bz_\by, \bc)p_{\psi}(\bz_\by|\bz_\bx, \bc)}{q_{\phi}(\bz_\by|\bz_\bx, \by, \bc)} \Bigg]\} \\
    &= \mathbb{E}_{q_{\phi}(\bz|\bx, \bc)}[\log p_{\theta}(\bx |  \bz_\bx, \by, \bc)] \\
    &\quad - \text{D}_{\text{KL}}[q_\phi (\bz_\bx | \bx, \bc) || p_{\psi}(\bz_\bx | \bc)] + \mathbb{E}_{q_{\phi}(\bz|\bx, \bc)}\{\\
    &\quad\log\mathbb{E}_{q_{\phi}(\bz_\by |\bz_\bx,\by, \bc)}\Bigg[\frac{p_{\theta}(\by|\bz_\by, \bc)p_{\psi}(\bz_\by|\bz_\bx, \bc)}{q_{\phi}(\bz_\by|\bz_\bx, \by, \bc)} \Bigg]\} \\
    &\geq \mathbb{E}_{q_{\phi}(\bz_\bx |\bx, \bc)}[\log p_{\theta}(\bx |  \bz_\bx, \by, \bc)] \\
    &\quad- \text{D}_{\text{KL}}[q_\phi (\bz_\bx | \bx, \bc) || p_{\psi}(\bz_\bx | \bc)]  \\
    &\quad+ \mathbb{E}_{q_{\phi}(\bz_\bx |\bx, \bc)} \{ \mathbb{E}_{q_{\phi}(\bz_\by |\bz_\bx, \by, \bc)}[\log p_\theta (\by | \bz_\by, \bc)] \\
    &\quad- \text{D}_{\text{KL}}[q_{\phi}(\bz_\by | \bz_\bx, \by,\bc)|| p_\psi (\bz_\by | \bz_\bx, \bc)] \} \\
    &\approx \mathbb{E}_{q_{\phi}(\bz_\bx |\bx, \bc)}[\log p_{\theta}(\bx |  \bz_\bx, \by, \bc)] \\
    &\quad- \text{D}_{\text{KL}}[q_\phi (\bz_\bx | \bx, \bc) || p_{\psi}(\bz_\bx | \bc)] \\
    &\quad+ \mathbb{E}_{q_{\phi}(\bz_\by |\bz_\bx, \by, \bc)}[\log p_\theta (\by | \bz_\by, \bc)] \\
    &\quad- \text{D}_{\text{KL}}[q_{\phi}(\bz_\by | \bz_\bx, \by,\bc)|| p_\psi (\bz_\by | \bz_\bx, \bc)]
\end{align*}
\end{proof}

\begin{table*}
	\small
	\centering
	\begin{tabular}{llll}
		\midrule[0.8pt]
		\textbf{Datasets} & \textbf{Train (\#)} & \textbf{Valid (\#)} & \textbf{Source}  \\
		\midrule[0.8pt]
		{SQuAD} & {86,588} & {10,507} & {Crowd-sourced questions from Wikipedia paragraph}  \\
		\midrule[0.8pt]
		{Natural Questions} & {104,071} & {12,836} & {Questions from actual userfor searching Wikipedia paragraph}  \\
		\midrule[0.8pt]
		{TriviaQA} & {74,160} & {7,785} & {Question and answer pairs authored by trivia enthusaists from the Web}  \\
		\midrule[0.8pt]
		{HarvestQA} & {1,259,691} & {-} & {Generated by neural networks from top-ranking 10,000 Wikipedia articles}  \\
		\midrule[0.8pt]
	\end{tabular}
	\caption{The statistics and the data source of SQuAD, Natural Questions, TriviaQA, and HarvestingQA.}
	\label{datastats}
\end{table*}

\begin{table}[htb]
\centering
\begin{tabular}{lll}
\midrule[0.8pt]
{\bf Replace} & {\bf EM} & {\bf F1} \\
\midrule[0.8pt]
{F1 $\leq 0.0$} & {82.4} & {89.39} \\
{F1 $\leq 20.0$} & {83.11} & {89.65} \\
{F1 $\leq 40.0$} & {\bf 83.32} & {\bf 89.79} \\
{F1 $\leq 60.0$} & {83.20} & {89.78} \\
{F1 $\leq 80.0$} & {83.09} & {89.75} \\
\midrule[0.8pt]
\end{tabular}
\caption{The effect of F1-based replacement strategy in semi-supervised setting of SQuAD+H$\times$100\%. All results are the performance on validation set of \citet{semantic}.}
\label{cross-val}
\end{table}

\subsection{Datatset}

The statistics and the data resource are summarized in Table~\ref{datastats}.

\noindent
\textbf{SQuAD}
We tokenize questions and contexts with WordPiece tokenizer from BERT. To fairly compare our proposed methods with the existing semi-supervised QA, we follow \citet{semantic}'s split, which divides original development set from SQuAD v1.1 \cite{squad1} into new validation set and test set. We adopt most of the codes from \citet{huggingface} for preprocessing data, training, and evaluating the BERT-base QA model.

\noindent
\textbf{Natural Questions}
Other than the original Natural Questions \cite{naturalquestion} dataset, we use subset of the dataset provided by MRQA shared task \cite{mrqa} for extractive QA. As semi-supervised setting with SQuAD, we split the validation set provided from MRQA into half for validation set and the others for test set. All the tokens from question and context are tokenized with WordPiece tokenizer from BERT. We generate QA pairs from context not containing html tag, and evaluate QA model with the official MRQA evaluation scripts. 

\noindent
\textbf{TriviaQA}
For TriviaQA \cite{triviaqa}, we also use the training set from MRQA shared task, and divide the development set from MRQA into half for validation set and the other for test set. All the tokens from question and context are tokenized with WordPiece tokenizer from BERT. For evaluation, we follow the MRQA's official evaluation procedure.  

\noindent
\textbf{HarvestingQA\footnote{\url{https://github.com/xinyadu/harvestingQA}}}
We use paragraphs from HarvestingQA dastaset \cite{harvesting} to generate QA pairs for QA-based Evaluation (QAE) and Reverse QA-based Evaluation (R-QAE). For the baseline QG models such as Maxout-QG and Semantic-QG, we use the same answer spans from the dataset. For the experiments of Maxout-QG baseline, we train the model and generate new questions from the context and answer, while the questions generated by Semantic-QG are provided by the authors \cite{semantic}.

\subsection{Training Details}
\label{training-detail}
\noindent
\textbf{Maxout-QG}
We use Adam \cite{adam} optimizer with the batch size of 64 and set the initial learning rate of $10^{-3}$. We always set the beam size of 10 for decoding. We also evaluate the Maxout-QG model on our SQuAD validation set with BLEU4 \cite{bleu}, and get 15.68 points.

\noindent
\textbf{Selection of Threshold for Replacement}
As mentioned in our paper, we use the threshold of $40.0$ selected via cross-validation of the QA model performance, using both the full SQuAD and HarvestingQA dataset for QAG. The detailed selection processes are as follows: 1) train QA model on only human annotated data, 2) compute F1 score of generated QA pairs, and 3) if the F1 score is lower than the threshold, replace the generated answer with the prediction of QA model. We investigate the optimal value of threshold among $[20.0, 40.0, 60.0, 80.0]$ using our validation set of SQuAD. Table~\ref{cross-val} shows the results of cross-validation on the validation set. The optimal value of $40.0$ is used for semi-supervised experiments on Natural Questions and TriviaQA. For fully unlabeled semi-supervised experiments on Natural Questions and TriviaQA, the QA model is only trained on SQuAD and used to replace the synthetic QA pairs (denoted in our paper as N$\times$1-10, T$\times$1-10).

\noindent
\textbf{Semi-supervised learning}
For the semi-supervised learning experiment on SQuAD, we follow \citet{semantic}'s split for a fair comparison. Specifically, we receive the unique IDs for QA pairs from the authors and use exactly the same validation and test set as theirs. For the Natural Questions and TriviaQA experiments, we use our own split as mentioned in the above. We generate QA pairs from the paragraphs of Wikipedia extracted by \citet{harvesting} and train BERT-base QA model with the synthetic data for two epochs. Then we further train the model with human-annotated training data for two more epochs. The catastrophic forgetting reported in \citet{semantic} does not occur in our cases. We use Adam optimizer \cite{adam} with batch size 32 and follow the learning rate scheduling as described in \cite{bert} with initial learning rate  $2\cdot10^{-5}$ and $3\cdot10^{-5}$ for synthetic and human annotated data, respectively.

\subsection{Qualitative Examples}
\label{qualitativeexamples}
The qualitative examples in Table  \ref{appendix-qas}, \ref{appendix-one-to-many}, \ref{appendix-interporlation}  are shown in the next page.

\begin{table*}[t]
	\small
	\centering
	\begin{tabular}{l}
		\midrule[0.8pt]
		\midrule[0.8pt]
		\textbf{Paragraph-1} Near Tamins-Reichenau the Anterior Rhine and the Posterior Rhine join and form the Rhine. \ldots This section \\
		is nearly 86km long, and descends from a height of 599m to 396m. It flows through a wide glacial alpine valley known as \\
		the Rhine Valley (German: Rheintal). Near Sargans a natural dam, only a few metres high, \ldots The Alpine Rhine begins \\
		in the most western part of the Swiss canton of Graubünden, \ldots \\
		\midrule[0.8pt]
		{\textbf{Q-1}:\: how long is the rhine?} \\
		{\textbf{A-1}:\: 86km long} \\
		\midrule[0.8pt]
		{\textbf{Q-2}:\: how large is the dam?} \\
		{\textbf{A-2}:\: a few metres high} \\
		\midrule[0.8pt]
		{\textbf{Q-3}:\: where does the anterior rhine and the posterior rhine join the rhine?} \\
		{\textbf{A-3}:\: Tamins-Reichneau} \\
		\midrule[0.8pt]
		{\textbf{Q-4}:\: what type of valley does the rhine flows through?} \\
		{\textbf{A-4}:\: glacial alpine} \\
		\midrule[0.8pt]
		{\textbf{Q-5}:\: what is the rhine valley in german?} \\
		{\textbf{A-5}:\: Rheintal} \\
		\midrule[0.8pt]
		{\textbf{Q-6}:\: where deos the alpine rhine begin?} \\
		{\textbf{A-7}:\: Swiss canton of Graubünden} \\
		\midrule[0.8pt]
		\midrule[0.8pt]
		\textbf{Paragraph-2} Victoria is the centre of dairy farming in Australia. It is home to 60\% of Australia's 3 million dairy cattle \\
		and produces nearly two-thirds of the nation's milk, almost 6.4 billion litres. The state also has 2.4 million beef cattle, with \\
		more than 2.2 million cattle and calves slaughtered each year. In 2003–04, Victorian commercial fishing crews and \\
		aquaculture industry produced 11,634 tonnes of seafood valued at nearly \$109 million. \ldots \\
		\midrule[0.8pt]
		{\textbf{Q-1}:\: what industry produced 11,63 million tonnes of seafood in 2003-04 ?} \\
		{\textbf{A-1}:\: aquaculture} \\
		\midrule[0.8pt]
		{\textbf{Q-2}:\: what type of cattle is consumed in Victoria?} \\
		{\textbf{A-2}:\: beef} \\
		\midrule[0.8pt]
		{\textbf{Q-3}:\: in what year did victorian commercial fishing and aquaculture industry produce a large amount of seafood?} \\
		{\textbf{A-3}:\: 2003–04} \\
		\midrule[0.8pt]
		{\textbf{Q-4}:\: how many cattle and calves each year are slaughtered annually?} \\
		{\textbf{A-4}:\: 2.2 million} \\
		\midrule[0.8pt]
		{\textbf{Q-5}:\: how much of the nation's milk is produced by the dairy?} \\
		{\textbf{A-5}:\: two-thrids} \\
		\midrule[0.8pt]
		\midrule[0.8pt]
		\textbf{Paragraph-3} A teacher's role may vary among cultures. Teachers may provide instruction in literacy and numeracy, \\
		craftsmanship or vocational training, the arts, religion, civics, community roles, or life skills. \\
		\midrule[0.8pt]
		{\textbf{Q-1}:\: what do a teacher's role vary?} \\
		{\textbf{A-1}:\: culture} \\
		\midrule[0.8pt]
		{\textbf{Q-2}:\: what do teachers provide instruction in?} \\
		{\textbf{A-2}:\: vocational training} \\
		\midrule[0.8pt]
		{\textbf{Q-3}:\: what is one thing a teacher may provide instruction for?} \\
		{\textbf{A-3}:\: community roles} \\
		\midrule[0.8pt]
		{\textbf{Q-4}:\: what is one of the skills that teachers provide in?} \\
		{\textbf{A-4}:\: life skills} \\
		\midrule[0.8pt]
		\midrule[0.8pt]
	\end{tabular}
	\caption{Examples of QA pairs generated by our Info-HCVAE. We sample multiple latent variables from $p_{\psi}(\cdot)$, and feed them to generation networks. All the paragraphs are from validation set of SQuAD.}
	\label{appendix-qas}
\end{table*}

\begin{table*}[t]
	\small
	\centering
	\begin{tabular}{l}
		\midrule[0.8pt]
		\midrule[0.8pt]
		\textbf{Paragraph-1} Super bowl 50 was an american football game to determine the champion of the National Football League  \\
		(NFL) for the 2015 season. The American Football Conference (AFC) champion \hlpink{Denver Broncos} defeated the National \\ 
		Football Conference (NFC) champion Carolina Panthers 24 – 10 to earn their third super bowl title. \ldots \\
		\midrule[0.8pt]
		\textbf{GT}\: which NFL team represented the AFC at super bowl 50? \\
		\midrule[0.8pt]
		\textbf{Ours-1}\: what team did the American Football Conference represent?\\
		\textbf{Ours-2}\: who won the 2015 American Football Conference? \\
		\textbf{Ours-3}\: which team defeated the carolina panthers? \\
		\textbf{Ours-4}\: who defeated the panthers in 2015? \\
		\textbf{Ours-5}\: what team defeated the carolina panthers in the 2015 season? \\
		\textbf{Ours-6}\: who was the champion of the American Football League in the 2015 season? \\
		\textbf{Ours-7}\: what team won the 2015 American Football Conference? \\
		\midrule[0.8pt]
		\midrule[0.8pt]
		\textbf{Paragraph-2} \ldots Some clergy offer healing services, while exorcism is an occasional practice by some clergy in the united \\
		methodist church in \hlpink{Africa}. \ldots \\
		\midrule[0.8pt]
		\textbf{GT}\: in what country does some clergy in the umc occasionally practice exorcism?\\
		\midrule[0.8pt]
		\textbf{Ours-1}\: in what country do some clergy in the united methodist church take place?\\
		\textbf{Ours-2}\: in what country is exorcism practice an occasional practice? \\
		\textbf{Ours-3}\: use of exorcism is an occasional practice in what country?\\
		\textbf{Ours-4}\: is exorcism usually an occasional practice in what country? \\
		\midrule[0.8pt]
		\midrule[0.8pt]
		\textbf{Paragraph-3} \ldots, the city was the subject of a song , ``walking into fresno'' , written by hall of fame guitarist \hlpink{Bill Aken} \ldots \\
		\midrule[0.8pt]
		\textbf{GT}\: who wrote ``walking in fresno''? \\
		\midrule[0.8pt]
		\textbf{Ours-1}\: who wrote ``walking into fresno''?\\
		\textbf{Ours-2}\: ``walking into fresno'' was written by whom?\\
		\textbf{Ours-3}\: the song ``walking into fresno'' was written by whom?\\
		\midrule[0.8pt]
		\midrule[0.8pt]
	\end{tabular}
	\caption{Examples of \textit{one-to-many} mapping of our Info-HCVAE. Answers are highlighted by pink. We sample multiple question latent variables from $p_{\psi}(\rvzx|\rvc)$, and feed them to question generation networks with a fixed answer. \textbf{GT} denotes ground-truth question, and \textbf{Seq2Seq} denotes question generated by Maxout-QG. All the paragraphs, ground truth questions, and answers are from validation set of SQuAD.}
	\label{appendix-one-to-many}
\end{table*}

\def\limitarrow#1{%
	\begin{tikzpicture}
	\draw[<-] (0,1.25) to (0,0.5);
	\node at (0,0) {#1};
	\draw[<-] (0,-1.25) to (0,-0.5);
	\end{tikzpicture}}
\begin{table*}[t]
	\small
	\centering
	\begin{tabular}{cl}
	    \midrule[0.8pt]
		\midrule[0.8pt]
		\multicolumn{2}{l}{\textbf{Paragraph-1} Notre Dame is known for its competitive admissions, with the incoming class enrolling in fall 2015 admitting} \\ 
		\multicolumn{2}{l}{\hlyellow{3,577 from a pool of 18,156 (19.7\%)}. The academic profile of the enrolled class continues to rate among \hlpink{the top 10 to 15}} \\ 
		\multicolumn{2}{l}{in the nation for national research universities. \ldots 1,400 of the \hlcyan{3,577} (39.1\% ) were admitted under the early action plan.}\\
		\midrule[0.8pt]
		\multirow{2}{*}{\bf Ori1} & {{\bf Q} where does notre dame rank in terms of academic profile among research universities in the us?} \\ 
		&{{\bf A} the top 10 to 15 in the nation} \\
		\midrule[0.8pt]
		\multirow{8}{*}{\limitarrow{\bf Gen}} & \textit{{\bf Q} where does the academic profile of notre dame rank?} \\ 
		& \textit{{\bf A} \hlpink{the top 10 to 15}} \\
		\cmidrule[0.8pt]{2-2}
		& \textit{{\bf Q} what was the rate of the incoming class enrolling in the fall of 2015?} \\
		& \textit{{\bf A} \hlyellow{3,577 from a pool of 18,156 (19.7\%)}}\\
		\cmidrule[0.8pt]{2-2}
		& \textit{{\bf Q} how many students attended notre dame?} \\
		& \textit{{\bf A} \hlcyan{3,577}} \\
		\cmidrule[0.8pt]{2-2}
		\midrule[0.8pt]
		\multirow{2}{*}{\bf Ori2} & {{\bf Q} what percentage of students at notre dame participated in the early action program?} \\
		& {{\bf A} 39.1\%}\\
		\midrule[0.8pt]
		\midrule[0.8pt]
		\multicolumn{2}{l}{\textbf{Paragraph-2} \ldots begun as a one-page journal in September \hlpink{1876}, the scholastic magazine is issued \hlcyan{twice} monthly and \ldots} \\ 
		\multicolumn{2}{l}{In \hlyellow{1987}, when some students believed that the observer began to show a \ldots In spring 2008 an undergraduate journal for}\\
		\multicolumn{2}{l}{political science research, beyond politics, made its debut.} \\ 
		\midrule[0.8pt]
		\multirow{2}{*}{\bf Ori1} & {{\bf Q} when did the scholastic magazine of notre dame begin publishing?} \\ 
		&{{\bf A} september 1876} \\
		\midrule[0.8pt]
		\multirow{8}{*}{\limitarrow{\bf Gen}} & \textit{{\bf Q} when was the scholastic magazine published?} \\ 
		& \textit{{\bf A} \hlpink{1876}} \\
		\cmidrule[0.8pt]{2-2}
		& \textit{{\bf Q} in what year did notre dame get its liberal newspaper?} \\
		& \textit{{\bf A} \hlyellow{1987}}\\
		\cmidrule[0.8pt]{2-2}
		& \textit{{\bf Q} how often is the scholastic magazine published ?} \\
		& \textit{{\bf A} \hlcyan{twice}} \\
		\cmidrule[0.8pt]{2-2}
		\midrule[0.8pt]
		\multirow{2}{*}{\bf Ori2} & {{\bf Q} in what year did notre dame begin its undergraduate journal ?} \\
		& {{\bf A} 2008}\\
		\midrule[0.8pt]
		\midrule[0.8pt]
		\multicolumn{2}{l}{\textbf{Paragraph-3} As at most other universities, notre dame's students run \hlyellow{a number of news media outlets}. \hlpink{The nine student}}\\
		\multicolumn{2}{l}{\hlpink{- run outlets} include \ldots, and \hlgreen{several} magazines and journals. \ldots. the dome yearbook is published \hlcyan{annually}. \ldots}\\
		\midrule[0.8pt]
		\multirow{2}{*}{\bf Ori1} & {{\bf Q} what is the daily student paper at notre dame called?} \\ 
		&{{\bf A} the observer} \\
		\midrule[0.8pt]
		\multirow{8}{*}{\limitarrow{\bf Gen}} & \textit{{\bf Q} how many student media outlets are there at notre dame?} \\ 
		& \textit{{\bf A} \hlpink{nine student - run outlets include three}} \\
		\cmidrule[0.8pt]{2-2}
		& \textit{{\bf Q} what type of media is the student paper at notre dame?} \\
		& \textit{{\bf A} \hlyellow{a number of news media}}\\
		\cmidrule[0.8pt]{2-2}
		& \textit{{\bf Q} how often is the dome published?} \\
		& \textit{{\bf A} \hlcyan{annually}} \\
		\cmidrule[0.8pt]{2-2}
		& \textit{{\bf Q} how many magazines are published at notre dame ?} \\
		& \textit{{\bf A} \hlgreen{several}} \\
		\midrule[0.8pt]
		\multirow{2}{*}{\bf Ori2} & {{\bf Q} how many student news papers are found at notre dame ?} \\
		& {{\bf A} three}\\
		\midrule[0.8pt]
	\end{tabular}
	\caption{QA pairs generated by interpolating between two latent codes encoded by our posterior networks. \textbf{Ori1} and \textbf{Ori2} are from training set of SQuAD.}
	\label{appendix-interporlation}
\end{table*}


\end{document}